\documentclass[conference, 10pt, twocolumn]{ieeeconf}       
\bstctlcite{bstctl:etal, bstctl:nodash, bstctl:simpurl}
\def\BibTeX{{\rm B\kern-.05em{\sc i\kern-.025em b}\kern-.08em
    T\kern-.1667em\lower.7ex\hbox{E}\kern-.125emX}}
    
\usepackage[left=54pt, right=54pt,bottom=54pt, top=54pt]{geometry}

\usepackage{amsmath,mathrsfs,amsfonts,amssymb,graphicx,epsfig}
 \usepackage{amsthm}
\usepackage{subcaption}
\usepackage{color,multirow,rotating}
\usepackage{algorithm,algpseudocode,algorithmicx}
\usepackage{cite,url,framed,bm,balance,nicematrix,physics}
\usepackage{stmaryrd,mathtools}

\newtheorem{theorem}{Theorem}

\newtheorem{remark}{Remark}

\usepackage{tikz}
\usetikzlibrary{calc,trees,positioning,arrows,chains,shapes.geometric,decorations.pathreplacing,decorations.pathmorphing,shapes, matrix,shapes.symbols}

\DeclareMathOperator*{\argmax}{arg\:max}

\newcommand{\prox}{{\rm{prox}}}

\makeatletter
\newcommand{\pushright}[1]{\ifmeasuring@#1\else\omit\hfill$\displaystyle#1$\fi\ignorespaces}

\newcommand{\dotminus}{\mathbin{\text{\@dotminus}}}

\newcommand{\@dotminus}{%
  \ooalign{\hidewidth\raise1ex\hbox{.}\hidewidth\cr$\m@th-$\cr}%
}

\allowdisplaybreaks

\title{\LARGE\textbf{Learning Concave Bid Shading Strategies in Online Auctions via Measure-valued Proximal Optimization
}}

\author{Iman Nodozi, Djordje Gligorijevic, and Abhishek Halder
\thanks{Iman Nodozi is with onsemi, CA, USA, {\tt\footnotesize iman.nodozi@onsemi.com}.}
\thanks{Djordje Gligorijevic is with Meta Platforms, Inc., CA, USA {\tt\footnotesize{gligorijevic@meta.com}}.}
\thanks{Abhishek Halder is with the Department of Aerospace Engineering, Iowa State University, Ames, IA, USA, {\tt\footnotesize{ahalder@iastate.edu}}.}
}

\IEEEoverridecommandlockouts
\begin{document}
\bstctlcite{IEEE_b:BSTcontrol}
\maketitle
\thispagestyle{empty}
\pagestyle{empty}

\begin{abstract}
    This work proposes a bid-shading strategy for first-price auctions as a measure-valued optimization problem. We consider a standard parametric form for bid shading and formulate the problem as convex optimization over the joint distribution of shading parameters. After each auction, the shading parameter distribution is adapted via a regularized Wasserstein-proximal update with a data-driven energy functional. This energy functional is conditional on the context, i.e., on publisher/user attributes such as domain, ad slot type, device, or location. The proposed algorithm encourages the bid distribution to place more weight on values with higher expected surplus, i.e., where the win probability and the value gap are both large. We show that the resulting measure-valued convex optimization problem admits a closed form solution. A numerical example illustrates the proposed method.
\end{abstract}

\section{Introduction}

We consider an advertising (ad) campaign managed by a demand-side platform, where multiple advertisers participate in real-time bidding for ad impressions. Let $\Omega$ denote the set of all impression opportunities $i$ made available for sale across one or more real-time marketplaces. Each impression $i \in \Omega$ is sold through a sealed auction mechanism, where bidders submit their bids independently, without knowledge of the competitors' bids.

There are two main types of cost models in online advertising: \emph{first-price} and \emph{second-price} auctions \cite[Ch. 2]{krishna2009auction}, \cite{edelman2007internet}. In a second-price auction, the highest bidder wins but pays the second-highest bid. In a first-price auction, the winner pays their own bid.  We focus on the first-price cost model. 

For each impression $i \in \Omega$, let $b_i$ denote the corresponding \emph{bid price}. The highest competing bid, denoted $B_i^*$, is unknown and modeled as a nonnegative random variable. The uppercase $B_i^*$ emphasizes its stochastic nature, while the lowercase $b_i^*$ refers to either a realized or expected value. The impression is awarded if $b_i \geq b_i^*$, in which case the incurred cost equals $b_i$. The cumulative distribution function (CDF) of $B_i^*$ is $F_{B_i^*}(b) := \mathbb{P}(B_i^* \leq b)$, assumed to be continuous with support on $\mathbb{R}_{\geq 0}$.



The \emph{intrinsic value} $v_i \in \mathbb{R}_{\geq 0}$ for the awarded impression is not known \emph{a priori} and must be estimated for each impression. Value estimation is based on campaign goals, such as the likelihood of user engagement (performance) or exposure (branding)\cite{wang2017deep,pan2018field,yoshikawa2018nonparametric}. Branding campaigns aim to increase brand awareness by reaching new users and building long-term recognition. In contrast, performance-driven campaigns focus on immediate outcomes such as maximizing clicks or purchases. In our setup, we assume that an \emph{estimated value} $\hat{v}_i$ is given, in order to focus on the learning of bid distributions and the optimization of bidding strategies under the first-price auction setting.

\subsubsection*{Related works} 
In a first-price auction, since the winner pays their own bid, \emph{bid shading strategies} \cite{sluis2019everything,gligorijevic2020bid,kasberger2024robust} are needed to avoid overpaying. Several algorithms have been proposed \cite{borgs2007dynamics,jank2011automated,karlsson2021adaptive} that rely on perturbation, forecasting, or pointwise maximization over explicitly estimated values or win probabilities. These approaches are often implemented within predefined traffic segments, which simplifies deployment but requires manual segment design and prevents sharing information across segments. For a recent survey, see \cite{ou2023survey}. Ref. \cite{kasberger2024robust} proposed a distributionally robust bid shading policy that maximizes surplus against worst-case value and competitor distributions chosen from Kullback-Leibler ambiguity set. 

In contrast, this work treats the learning task as a convex optimization problem over the space of joint probability measures/distributions on the shading parameters. Our method computes this parameter distribution via an entropy-regularized Wasserstein-proximal update, which simultaneously enforces distributional smoothness and incorporates observed auction feedback through a data-driven energy-like functional that we propose.

\subsubsection*{Contributions}
\begin{itemize}
    \item We reformulate bid shading as a convex optimization problem over a distribution of parameters, rather than estimating point values within pre-defined segments.  
    
    \item We solve this optimization problem over bid shading parameter distribution  using an entropy-regularized Wasserstein-proximal update. We derive a closed form solution for this update, making the approach practical for online use.  
    
    \item The update mechanism naturally shifts probability mass toward parameter values with higher expected surplus, where both win probability and value gap are large.  
\end{itemize}

\subsubsection*{Organization} In Sec. \ref{sec:formulation}, we explain the underlying models and ideas leading to the problem formulation. In Sec. \ref{sec:ProxUpdate}, we detail the proposed measure-valued proximal update, its closed form solution (Theorem \ref{Thm:WassProxOfLinear}), and the resulting Algorithm \ref{algoirithm1}. Sec. \ref{sec:numerical} provides illustrative numerical results for the proposed method. Sec. \ref{sec:conclusions} concludes the work.


\section{Problem Formulation}\label{sec:formulation}
We consider the problem of maximizing the expected total advertising value produced over the campaign, subject to a finite advertising budget $\xi$ that constrains the expected cumulative cost of winning impressions.

The bid prices $b_i \in \mathbb{R}$ serve as decision variables that control how resources are allocated for participating in the auctions, as per the following constrained optimization problem:
\begin{equation}
   \begin{array}{cl}
\max\limits _{\left\{b_i \in \mathbb{R} \mid \forall i \in \Omega\right\}} & \mathbb{E}[V] \\
\text {subject to:} & \mathbb{E}[C] \leq \xi,
\end{array} 
\end{equation}
where $\mathbb{E}[V]$ is the expected total value of awarded impressions, and $\mathbb{E}[C]$ is the expected cumulative cost. These quantities are defined as:

\begin{equation}
    V = \sum_{i \in \Omega} v_i \, \mathbb{I}_{\left\{b_i \geq B_i^*\right\}}, \quad 
C = \sum_{i \in \Omega} b_i \, \mathbb{I}_{\left\{b_i \geq B_i^*\right\}},
\label{defVandC}
\end{equation}
where the indicator function $\mathbb{I}_{\left\{b_i \geq B_i^*\right\}}=1$ if the impression $i$ wins (i.e., $b_i \geq B_i^*$), and $=0$ otherwise. 


Following the formulation in \cite{karlsson2021adaptive}, we consider a large-scale online bidding problem in first-price auctions, where the number of impression opportunities $|\Omega|$ is on the order of millions or billions, and the distribution of the highest competing bid $B_i^*$ is unknown. To address this, the bidding process is decomposed into three components: impression valuation, campaign-level budget control, and bid shading. 

\begin{figure*}[t]
    \centering
    \includegraphics[width=.8\textwidth]{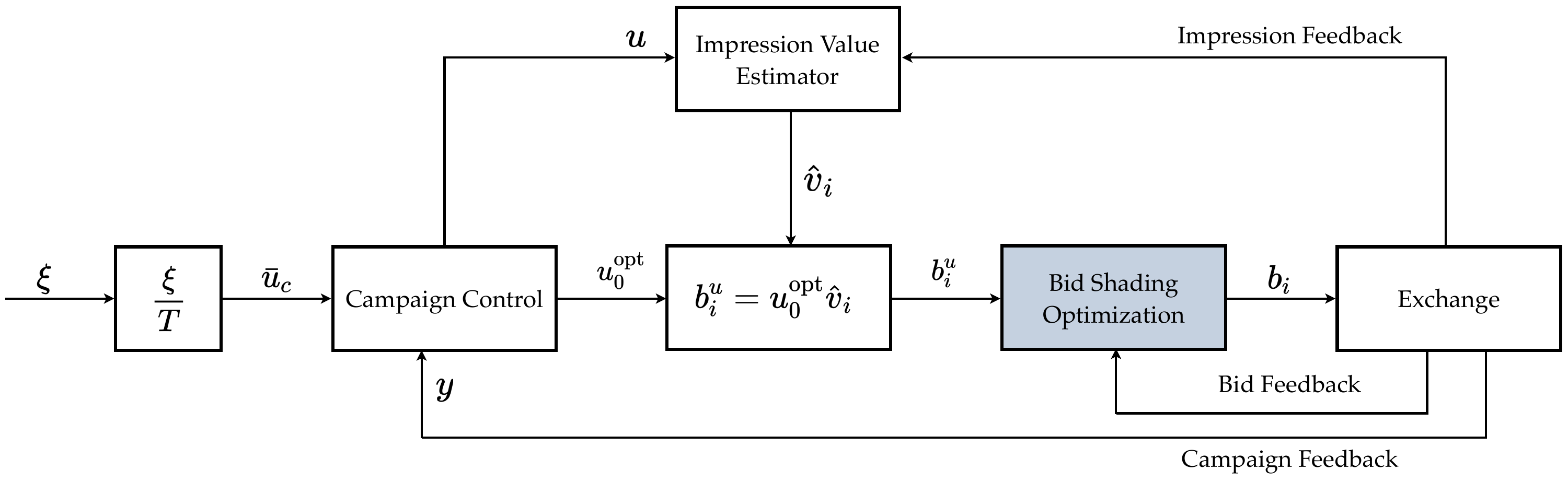}
    \caption{Closed-loop feedback control architecture for online advertising. The bid shading optimization method proposed in this work is for the highlighted block.}
    \label{fig:closed_loop}
\end{figure*}

Fig.~\ref{fig:closed_loop} illustrates a closed-loop feedback control architecture for this process. It captures the dynamic interactions between campaign-level control, impression valuation, bid shading, and the ad exchange. In this work, we focus on the \emph{Bid Shading Optimization} module (highlighted), which adapts the bid $b_i$ based on estimated valuations $\hat{v}_i$ and the optimal control signal $u_0^{\mathrm{opt}}$ (inverse of Lagrange multiplier). The remaining modules---campaign control, valuation, and feedback processing---are treated as given, with dynamics informed by prior literature \cite{karlsson2021adaptive,karlsson2024optimization,he2019identification,karlsson2014adaptive,karlsson2020feedback}. In particular, the campaign-level budget control is implemented as a feedback mechanism that adjusts spending over time. To ensure delivery of the campaign budget throughout the campaign horizon, the control signal \(u_0^{\mathrm{opt}}(t)\) is dynamically adjusted. 

The reference input \(\bar{u}_c(t)\) is injected into the campaign control block (see Fig.~\ref{fig:closed_loop}), representing the desired per-period spend rate. The total budget \(\xi\) is divided across \(T\) sample periods to define the setpoint 
$\bar{u}_c(t) = \xi/T.$
We consider a nonlinear, time-varying, and stochastic plant model. This model is an adjusted version of those presented in prior work~\cite{karlsson2014adaptive, karlsson2020feedback, karlsson2024optimization}, and is designed to capture key features of the auction environment while simplifying some internal components. Specifically, we simulate the observed spend dynamics as:
\begin{equation}
    y(t+1) = (1 + h(t)) \left(b_i(t) \mathbb{I}_{\{y_i(t) = 1\}}\right) (1 + w(t)), \label{actualSpend}
\end{equation}
\noindent where \(b_i(t)\) is the final bid submitted at time \(t\), \(y_i(t) \in \{0,1\}\) indicates whether the bid was successful, \(h(t)\) models periodic seasonality, and \(w(t)\) is multiplicative zero-mean noise. The bid \(b_i(t)\) is generated by the bid shading module based on the unshaded bid 

\begin{equation}
    b_i^u(t) = u_0^{\mathrm{opt}}(t)\hat{v}_i, \label{unshadedbid}
\end{equation}
where \(\hat{v}_i\) is the estimated impression value, and \(u_0^{\mathrm{opt}}(t)\) is the campaign-level control signal.

Our main contribution lies in formulating the bid shading problem as a convex optimization over the space of probability measures, replacing the explicit estimators and greedy maximization steps used in prior work~\cite{karlsson2021adaptive}. In contrast to prior methods, the bid distribution in our framework is iteratively computed using an entropy-regularized Wasserstein-proximal update, and we show that each such update admits a closed form solution. The latter facilitates practical implementation. 

Before introducing the proposed approach, we provide a brief overview of the baseline bid shading architecture.
\begin{theorem}\cite[Theorem 1]{karlsson2021adaptive}
Let $F_{B_i^{*}}$ denote the CDF of the highest competing bid for impression $i \in \Omega$. Then, the optimal bidding strategy corresponds to a pair $(u_0^{\mathrm{opt}}, b_i^{\mathrm{opt}})$ solving the coupled system:
\begin{align}
u_0^{\mathrm{opt}} &= \argmax\limits_{u_0 \geq 0} \left(\mathbb{E}[C] - \xi\right)^2, \label{eq:u0_opt} \\
b_i^{\mathrm{opt}} &= \argmax\limits_{b_i \geq 0} \left(u_0^{\mathrm{opt}} v_i - b_i\right) F_{B_i^{*}}(b_i), \quad \forall i \in \Omega, \label{eq:bi_opt}
\end{align}
where $\mathbb{E}[C]$ is the expected total cost induced by the bids $b_i$. If the objective in \eqref{eq:bi_opt} admits a unique maximizer for all $i\in\Omega, u_0 \in \mathbb{R}_{\geq 0}$, then the solution (i.e., the optimal bid shading strategy) is unique.
\end{theorem}

\begin{remark}
In the case of a \emph{second-price auction}, where the winner pays the highest competing bid rather than their own bid, the optimal strategy simplifies to $ b_i^{\mathrm{opt}} = u_0^{\mathrm{opt}} v_i,$
removing the need to consider the distribution of $B_i^{*}$ in the optimization. In our formulation, we refer to this quantity as the \emph{unshaded bid} $b_i^u$, which serves as the input to a bid shading process to determine the actual submitted bid.
\end{remark}

\noindent
The relationship between the unshaded bid and the actual submitted bid $b_i$ is commonly expressed in the literature via a parametric form \cite{karlsson2021adaptive, gligorijevic2020bid}:

\begin{equation}
b_i(\theta_1, \theta_2, \bm{a})= \begin{cases}\dfrac{\log \left(1+\theta_1 \theta_2 b_i^u(\bm{a})\right)}{\theta_2}, & \theta_2>0, \\ \theta_1 b_i^u(\bm{a}), & \theta_2=0,\end{cases}
   \label{shadingb}
\end{equation}
where $\theta_1 \in [0,1]$ and $\theta_2 \geq 0$ are parameters, and \(\bm{a}\in\mathbb{R}^{a}\) is a given publisher/user attribute. This function is continuous in both $\theta_1$ and $\theta_2$, and it is concave monotonically increasing in $b_i^u$ \eqref{unshadedbid}. The conventional approach, used as an industry benchmark, treats these as segment-specific shading parameters \cite{gligorijevic2020bid}. Segments are manually pre-defined (e.g., by exchange, domain, and device), and a separate set of parameters $(\theta_1, \theta_2)$ is identified for each segment, typically via a recursive least square or similar pointwise optimization \cite{karlsson2021adaptive}. In this work, we consider this parametric function \eqref{shadingb}, but instead of finding a single optimal pair $(\theta_1, \theta_2)$ for a pre-defined segment, we consider the joint probability measure or distribution over the parameter space, and compute this distribution as measure-valued proximal update. 

At each iteration, we sample candidate $(\theta_1,\theta_2)$ pairs from the current distribution and compute the bid $b_i$. After the impression is served, we observe the realized \emph{surplus} $S:=V-C$, where $V,C$ are as in \eqref{defVandC}. The surplus $S$ serves as the feedback signal to update the parameter distribution. This leads to a recursive, closed-form update that balances empirical performance with distributional regularity, without requiring the knowledge of the CDF $F_{B_i^{*}}$.

Next, we formalize this idea and detail the measure-valued proximal update, which governs how the parameter distribution evolves over time in response to observed outcomes.
 
\section{Measure-valued Proximal Update}\label{sec:ProxUpdate}
\subsection{Background} 
Let $\mathcal{P}_{2}(\mathbb{R}^{d})$ denote the space of Borel probability measures over $\mathbb{R}^{d}$ with finite second moments. Notice that $\mathcal{P}_2$ is not a vector space, but an infinite dimensional manifold.



In our context, we seek to find a probability measure $\mu(\theta_1,\theta_2)\in\mathcal{P}_{2}(\mathbb{R}^2)$ supported over the shading parameters $(\theta_1, \theta_2) \in \mathbb{R}^2$, that minimizes a suitable functional $\Phi(\mu)$ encoding both performance and regularity. This leads us to consider optimization problems of the form
\begin{align}
\underset{\mu \in \mathcal{P}_2(\mathbb{R}^2)}{\arg\min} \:\Phi(\mu),
\label{OptimizationMeasure}    
\end{align}
where the objective \( \Phi: \mathcal{P}_2(\mathbb{R}^2) \rightarrow \mathbb{R} \) is assumed to be proper lower semi-continuous and convex\footnote{along generalized geodesic w.r.t. the Wasserstein metric \cite{ambrosio2008gradient}} in $\mu$. 

\medskip
\noindent
\textbf{Wasserstein metric and its entropic regularization.} The \emph{squared Wasserstein distance} between a pair of probability measures $\mu_{\bm{x}},\mu_{\bm{y}}\in\mathcal{P}_{2}\left(\mathbb{R}^{d}\right)$, is
\begin{align}
W^{2}\left(\mu_x,\mu_y\right) := \underset{\pi\in\Pi\left(\mu_x,\mu_y\right)}{\min}\:\displaystyle\int_{\mathbb{R}^{2d}}c\left(\bm{x},\bm{y}\right)\:\differential\pi(\bm{x},\bm{y}),
\label{DefWassContinuous}    
\end{align}
where $\Pi\left(\mu_x,\mu_y\right)$ is the set of joint probability measures or couplings over the product space $\mathbb{R}^{2d}$, having $\bm{x}$ marginal $\mu_x$, and $\bm{y}$ marginal $\mu_{\bm{y}}$. We use the \emph{ground cost} $c\left(\bm{x},\bm{y}\right):=\|\bm{x}-\bm{y}\|_{2}^{2}$, the squared Euclidean distance in $\mathbb{R}^{d}$. It is well-known that $W$ is a metric over the manifold $\mathcal{P}_{2}(\mathbb{R}^{d})$.


The entropy-regularized variant of \eqref{DefWassContinuous}, called \emph{entropy-regularized squared Wasserstein distance}, is  
\begin{align}
&W_{\varepsilon}^{2}\left(\mu_x,\mu_y\right) :=\underset{\pi\in\Pi\left(\mu_x,\mu_y\right)}{\min}\bigg\{\displaystyle\int_{\mathbb{R}^{2d}}c\left(\bm{x},\bm{y}\right)\:\differential\pi(\bm{x},\bm{y})  \nonumber \\
& \qquad +\varepsilon \displaystyle\int_{\mathbb{R}^{2d}}\log\left(\dfrac{\differential\pi(\bm{x},\bm{y})}{\differential\mu_{\bm{x}}(\bm{x})\differential\mu_{\bm{y}}(\bm{y})}\right)\differential\pi(\bm{x},\bm{y})\bigg\},
\label{DefRegWassContinuous}    
\end{align}
where $\varepsilon > 0$ is the regularization parameter. Notice that $W_{\varepsilon}\geq 0$ and is symmetric w.r.t. $\mu_x,\mu_y$, but does not satisfy other axioms (triangle inequality, indiscenrability) of a metric \cite[p. 1]{burago2001course}. As expected, $W_{\varepsilon}\rightarrow W$ as $\varepsilon\downarrow 0$.  

Due to strict convexity w.r.t. $\pi$, the regularized variant \eqref{DefRegWassContinuous} is computationally more convenient than \eqref{DefWassContinuous}, and is widely used in measure-valued variational problems \cite{benamou2015iterative, carlier2017convergence, peyre2015entropic, cuturi2016smoothed,caluya2019gradient,caluya2021wasserstein,nodozi2023wasserstein}. We will use the same in the developments that follow.

\medskip
\noindent
\textbf{Wasserstein Proximal Operator.} Under the stated assumptions on the functional $\Phi$, the minimization \eqref{OptimizationMeasure} can be performed as a steepest descent w.r.t. the $W$ metric in \eqref{DefWassContinuous}. This is realized via the time-stepping
\begin{align}
\mu^k = \underbrace{\underset{{\mu \in \mathcal{P}_2(\mathbb{R}^{2})}}{\arg\min} \left[ \frac{1}{2} W^2(\mu, \mu^{k-1}) + h \Phi\right]}_{=:\prox^{W}_{h\Phi}\left(\mu^{k-1}\right)}, \quad k\in\mathbb{N},
\label{muRecursionGeneral}
\end{align}
where \( h > 0 \) is the step size. The sequence of measures $\{\mu^{k-1}\}_{k\in\mathbb{N}}$ generated by \eqref{muRecursionGeneral} is guaranteed to converge to the minimizer of \eqref{OptimizationMeasure}. We refer to $\prox^{W}_{h\Phi}\left(\cdot\right)$ as the \emph{Wasserstein-proximal operator}, and its fixed point is the minimizer of \eqref{OptimizationMeasure}. Such recursions appeared earlier in stochastic prediction \cite{caluya2019proximal,halder2020hopfield,halder2022stochastic} and filtering \cite{halder2017gradient,halder2018gradient,halder2019proximal}. They generalize the finite dimensional proximal recursion \cite{parikh2014proximal} and operator \cite{rockafellar1976monotone} $\prox^{\mathrm{dist}_{\bm{G}}}_{h\phi}\left(\cdot\right)$, realizing natural gradient descent \cite{amari1998natural} w.r.t. $\mathrm{dist}_{\bm{G}}$ induced by the underlying metric tensor $\bm{G}$. The case $\bm{G}\equiv\bm{I}$ is the Euclidean proximal update.

When $W^2$ in \eqref{muRecursionGeneral} is replaced with its entropy-regularized variant $W_{\varepsilon}^2$, we denote the corresponding proximal operator as $\prox^{W_{\varepsilon}}_{h\Phi}\left(\cdot\right)$. For computational convenience, we use $\prox^{W_{\varepsilon}}_{h\Phi}\left(\cdot\right)$ with small $\varepsilon>0$ as a proxy for $\prox^{W}_{h\Phi}\left(\cdot\right)$.


\subsection{Main Idea}

We propose a Wasserstein-proximal update of the form \eqref{muRecursionGeneral} for bid shading. Specifically, given publisher and user attributes\footnote{attributes may include publisher domain, content category, ad slot type (e.g., banner/video, above/below fold), and user-side features such as device, location, browser, demographics.} \(\bm{a}\in\mathbb{R}^{a}\), the advertiser at time step $k\in\mathbb{N}$, performs the update \eqref{muRecursionGeneral} with the functional \( \Phi\equiv \Phi^k(\mu \mid \bm{a}) \), an energy-like functional conditional on the attribute \(\boldsymbol{a}\). 

We define this functional at step \( k\in\mathbb{N} \) as
\begin{align}
&\Phi^k(\mu \mid \bm{a}) :=-\mathbb{E}_{(\theta_1, \theta_2) \sim \mu} \big[\hat{f}_k(\theta_1, \theta_2\mid \bm{a})\cdot\nonumber\\ &\qquad\qquad\qquad\qquad\left(\hat{v}_i^k(\theta_1, \theta_2) - b_i(\theta_1, \theta_2,\bm{a})\right) \big],
\label{EnergyFuntional}
\end{align}
which encourages the bid distribution to place more weight on values with higher expected surplus, i.e., where the win probability \( \hat{f}_k(\theta_1, \theta_2\mid\bm{a}) \) and value gap \( \big(\hat{v}_i^k(\theta_1, \theta_2) - b_i(\theta_1, \theta_2,\bm{a})\big) \) are both large. Intuitively, $\mu$ should be updated in a way that lowers the value of the functional $\Phi^{k}$ in \eqref{EnergyFuntional}. 

In this work, we model the win probability as
\begin{equation}
\begin{aligned}
 &\hat{f}_k(\theta_1, \theta_2 \mid \boldsymbol{a})
=
\sigma\!\Big(w_0 + \langle\boldsymbol{w},\boldsymbol{a}\rangle + \beta\, \log b_i(\theta_1,\theta_2,\boldsymbol{a})\Big),
\label{eq:winprob} 
\end{aligned}
\end{equation}
where
$\sigma(\cdot):=1/\left(1+\exp(-\cdot)\right)$, the standard logistic function. We train the parameters \((w_0,\boldsymbol{w},\beta)\in\mathbb{R}\times\mathbb{R}^{a}\times\mathbb{R}_{>0}\) via exponentially weighted logistic regression on \((\boldsymbol{a}_i,b_i,y_i)\).

The proposed Wasserstein-proximal update \eqref{muRecursionGeneral} is closely related to Wasserstein gradient flow \cite{ambrosio2008gradient}, \cite{santambrogio2017euclidean}. Specifically, for the functional \( \Phi^k: \mathcal{P}_2(\mathbb{R}^2) \rightarrow \mathbb{R} \), its \emph{Wasserstein gradient} \cite[Ch.~8]{ambrosio2008gradient} evaluated at \( \mu \in \mathcal{P}_2(\mathbb{R}^2) \), is
\begin{equation}
\nabla^W \Phi^k(\mu) := -\nabla \cdot \left( \mu \nabla \frac{\delta \Phi^k}{\delta \mu} \right),
\end{equation}
where \( \nabla \) denotes the Euclidean gradient and \( \frac{\delta}{\delta \mu} \) is the functional derivative of \( \Phi_k \). As the step size \( h \downarrow 0 \), the sequence $\{\mu^{k-1}(h)\}_{k\in\mathbb{N}}$ generated by \eqref{muRecursionGeneral} converges to the trajectory \( \mu(t, \cdot) \) governed by the Wasserstein gradient flow:
\begin{equation}
\frac{\partial \mu}{\partial t} = -\nabla^W \Phi^k(\mu), \quad \mu(t = 0, \cdot) = \mu^0(\cdot).
\label{eq:WassGradFlow}
\end{equation}
Thus, the proposed Wasserstein-proximal update can be interpreted as a time-discretized approximation of the Wasserstein gradient flow \eqref{eq:WassGradFlow}. Each proximal step performs an implicit Euler step in the metric space $\left(\mathcal{P}_2\left(\mathbb{R}^2\right),W\right)$.

\begin{remark}\label{RemarkPhikIslinearInMu}
A key observation is that the functional $\Phi^k$ in \eqref{EnergyFuntional} is linear in the bid shading parameter distribution $\mu\in\mathcal{P}_2\left(\mathbb{R}^{2}\right)$. Building on this, we show in Sec. \ref{subsec:Results} that the corresponding regularized Wasserstein-proximal update \( \mu^k \) in \eqref{muRecursionGeneral} can be computed in closed form.
\end{remark}


\subsection{Results}\label{subsec:Results}

Henceforth, we focus on discrete  state space with finite $N$ samples. So for each fixed $k \in \mathbb{N}$,
\begin{align}
\bm{\mu}^{k-1} \in \Delta^{N-1} := \left\{\bm{p} \in \mathbb{R}^{N} \;\middle|\; \bm{p} \geq \bm{0},\; \langle\bm{1},\bm{p}\rangle = 1 \right\}.
\label{DefDiscreteMuVector}
\end{align}
For $\bm{p}, \bm{q} \in \Delta^{N-1}$, let the \emph{transport polytope}
\begin{align}
   & \Pi_{N}\left(\bm{p}, \bm{q} \right)\!:=\!\{ \bm{M} \in \mathbb{R}^{N \times N} \mid \bm{M} \geq 0,\bm{M}\bm{1} = \bm{p},\bm{M}^{\top}\bm{1} = \bm{q}\},
\label{defTransportPolytope}    
\end{align}
and the \emph{Euclidean ground cost matrix} $\bm{C} \in \mathbb{R}_{\geq 0}^{N \times N}$ with entries $\bm{C}(i, j) := \|\bm{\theta}_i - \bm{\theta}_j\|_2^2$ for samples $\{\bm{\theta}_j\}_{j\in [N]} \subset \mathbb{R}^2$.

Then, the discrete version of the entropy-regularized Wasserstein-proximal update $\mu^{k}=\prox^{W_{\varepsilon}}_{h\Phi}\left(\mu^{k-1}\right)$ becomes
\begin{align}
&\bm{\mu}^{k} = \prox^{W_{\varepsilon}}_{h\Phi^k}\left(\bm{\mu}^{k - 1}\right)\nonumber \\
&=\underset{\bm{\mu}\in\Delta^{N-1}}{\arg\min}\bigg\{\underset{\bm{M}\in\Pi_{N}\left(\bm{\mu},\bm{\mu}^{k-1}\right)}{\min}\bigg\langle\frac{1}{2}\bm{C} + \varepsilon\log\bm{M},\bm{M}\bigg\rangle \nonumber \\
& \qquad \qquad \quad + h\Phi^k(\bm{\mu}\mid\bm{a})\bigg\}, \quad \forall k\in\mathbb{N}. \label{MuUpdateDiscreteRegularized}
\end{align}
Theorem \ref{Thm:WassProxOfLinear} next shows that when $\Phi^k(\bm{\mu}\mid\bm{a})$ is linear in $\bm{\mu}$, the $\bm{\mu}^{k}$ update in \eqref{MuUpdateDiscreteRegularized} can be computed in closed form.

\begin{theorem}\label{Thm:WassProxOfLinear} Let \( \Phi^k(\bm{\mu}) := \langle \bm{r}^{k}, \bm{\mu}\rangle \) for given \( \bm{r}^{k}\in \mathbb{R}^N\setminus\{\bm{0}\}, k\in\mathbb{N}\), defined \( \forall\bm{\mu} \in \Delta^{N-1} \). Given Euclidean ground cost matrix \(\bm{C}\in \mathbb{R}_{\geq 0}^{N \times N} \), and regularization parameter \( \varepsilon > 0 \), let 
$\bm{K} := \exp\left( -\bm{C}/(2\varepsilon)\right)$ where $\exp$ acts elementwise. Then, for any probability vector \( \bm{\mu}^{k - 1} \in \Delta^{N-1} \) and step size \( h > 0 \), the update \eqref{MuUpdateDiscreteRegularized} is given by
\begin{equation}
\begin{aligned}
&\bm{\mu}^{k} =\exp\!\left(-h\bm{r}^{k}/\varepsilon\right)\!\odot\! \left( \bm{K}^\top\!\!\left( \bm{\mu}^{k - 1}\!\oslash\!\left( \bm{K} \exp\left(-h\bm{r}^{k}/\varepsilon\right) \right)\right)\right),
\label{eq:ClosedFormWassProx}
\end{aligned}
\end{equation}
where \( \odot \) and \( \oslash \) denote elementwise multiplication and division, respectively.
\end{theorem}
\begin{proof}
For generic $\Phi^{k}$ convex in $\bm{\mu}\in\Delta^{N-1}$, it is known \cite[Lemma 3.5]{karlsson2017generalized},\cite[Theorem 1]{caluya2019gradient} that the entropy-regularized Wasserstein proximal update 
\begin{align}
&\bm{\mu}^{k} =\exp\left(h\bm{\lambda}_{1}^{\rm{opt}}/\varepsilon\right) \odot \left(\exp\left(-\bm{C}^{\top}/(2\varepsilon)\right) \exp\left(h\bm{\lambda}_{0}^{\rm{opt}}/\varepsilon\right)\right),
\label{MuUpdateGeneral}    
\end{align}
where the optimal dual variables $\bm{\lambda}_{0}^{\rm{opt}},\bm{\lambda}_{1}^{\rm{opt}}\in\mathbb{R}^{N}$ solve the coupled system
\begin{subequations}
\label{lambda0lambda1equations} 
\begin{align}
&\exp\!\left(h\bm{\lambda}_{0}^{\rm{opt}}/\varepsilon\right)\!\odot\!\left( \exp\!\left(-\bm{C}/(2\varepsilon)\right)\!\exp\!\left(h\bm{\lambda}_{1}^{\rm{opt}}/\varepsilon\right)\right) = \bm{\mu}_{k-1},  \label{ZetakEquation}\\ 
&\bm{0} \in \partial_{\bm{\lambda}_{1}^{\rm{opt}}}\Phi_k^{*}\left(-\bm{\lambda}_{1}^{\rm{opt}}\right) - \!\exp\!\left(h\bm{\lambda}_{1}^{\rm{opt}}/\varepsilon\right)\nonumber \\
&\qquad\qquad\odot\left(\exp\left(-\bm{C}^{\top}/(2\varepsilon)\right) \exp\left(h\bm{\lambda}_{0}^{\rm{opt}}/\varepsilon\right)\right), \label{ZeroInSubdifferential}  
\end{align} 
\label{CoupledSystemlambda0optlambda1opt}
\end{subequations}
wherein $\bm{0}$ denotes the $N\times 1$ column vector of zeros, $\partial_{\bm{\lambda}_{1}^{\rm{opt}}}$ denotes the subgradient, and the superscript $^{*}$ denotes the Legendre-Fenchel conjugate.

For $\Phi^k(\bm{\mu}) := \langle \bm{r}^{k}, \bm{\mu}\rangle$ with $\bm{r}^{k}\in \mathbb{R}^N\setminus\{\bm{0}\}$, its Legendre-Fenchel conjugate is an indicator function:
\begin{align}
\Phi_{k}^{*}(-\bm{\lambda}_{1}) = \begin{cases}
0 & \text{if}\quad\bm{\lambda}_{1} = -\bm{r}^{k},\\
+\infty & \text{otherwise}.
\end{cases} 
\label{LegendreFenchelConjugate}
\end{align}
So $\bm{\lambda}_{1}^{\rm{opt}} = -\bm{r}^{k}$, which combined with \eqref{ZetakEquation} gives 
\begin{align}
\exp\!\left(h\bm{\lambda}_{0}^{\rm{opt}}/\varepsilon\right) = \bm{\mu}_{k-1}\oslash\left(\bm{K}\!\exp\!\left(-h\bm{r}^{k}/\varepsilon\right)\right), \label{lambda0opt}    
\end{align}
where we used $\bm{K} := \exp\left( -\bm{C}/(2\varepsilon)\right)$.

Substituting \eqref{lambda0opt} back in \eqref{MuUpdateGeneral}, and using $\bm{\lambda}_{1}^{\rm{opt}} = -\bm{r}^{k}$ again, we arrive at \eqref{eq:ClosedFormWassProx}.
\end{proof}

\begin{remark}
Notice that evaluating \eqref{eq:ClosedFormWassProx} is of complexity $\mathcal{O}(N^2)$, and involves vectorized operations (e.g., matrix-vector, not matrix-matrix, multiplication).     
\end{remark}

Since the functional $\Phi^k(\bm{\mu}\mid\bm{a})$ in \eqref{EnergyFuntional} is indeed of the form $\langle \bm{r}^{k}, \bm{\mu}\rangle$ with the $j$th component of $\bm{r}^{k}$ being 
\begin{align}
\!\left[\bm{r}^{k}\right]_{j} \!=\! \hat{f}_k(\theta_{1j}^{k}, \theta_{2j}^{k}\mid \!\bm{a})\!\!\left(\hat{v}_i^k(\theta_{1j}^{k}, \theta_{2j}^{k}) - b_i(\theta_{1j}^{k}, \theta_{2j}^{k},\bm{a})\right)
\label{Ourrk}    
\end{align}
where the sample index $j\in[N]$, the $\hat{f}_{k}$ is given by \eqref{eq:winprob}, and $b_i$ is given by \eqref{shadingb}. As a result, the closed form \eqref{eq:ClosedFormWassProx} in Theorem \ref{Thm:WassProxOfLinear} applies. Then, the proposed entropy-regularized Wassertein-proximal update for the bid shading parameter distribution, can be implemented for each $k\in\mathbb{N}$, as outlined in Algorithm \ref{algoirithm1}.
\begin{algorithm}[H]
\caption{Compute bid shading parameter distribution via regularized Wasserstein-proximal update}
\begin{algorithmic}[1]
\State \textbf{Initialize} parameter distribution $\bm{\mu}^0 \in \Delta^{N-1}$
\For{$k = 1$ to $T$}
    \State Sample parameters $(\theta_1^{k}, \theta_2^{k}) \sim \bm{\mu}^{k-1}$
    \State Apply bid and observe outcome in $ \{0, 1\}$
    \State Update the estimated win probability $\hat{f}_k(\theta_1^{k}, \theta_2^{k} \mid \bm{a})$ using \eqref{eq:winprob} with new observed outcome
    \State Compute $\bm{r}^{k}$ using \eqref{Ourrk}
    \State $\bm{\mu}^k \gets \prox^{W_{\varepsilon}}_{h\Phi^k}(\bm{\mu}^{k - 1})$\!\! \Comment{update distribution per \eqref{eq:ClosedFormWassProx}}
\EndFor
\end{algorithmic}\label{algoirithm1}
\end{algorithm}

\section{Numerical Experiments}\label{sec:numerical}
We consider an environment where the advertiser interacts with a dynamic market over a total horizon of 408 hours. Time is discretized in steps of \(\Delta = 1/30\) hours, corresponding to \(T_{\text{day}} = 720\) steps per day. The total number of simulation steps is thus \(T_{\text{total}} = 12240\).

To compute the unshaded bid, we use Algorithm~1 from~\cite{karlsson2024optimization}, with some modifications tailored to our setting. Specifically, we model the estimated valuation of the impression opportunity $\hat{v}_i$ as
\[
\hat{v}_i = u + {\mathrm{noise}}, \quad {\mathrm{noise}} \sim \mathcal{N}(0, 0.1^2),
\]
where $u$ is the control signal coming from the campaign controller, as illustrated in Figure \ref{fig:closed_loop}. The bid $b_i$ is then computed using~\eqref{shadingb}.

To mimic the dynamics of real-world auction environments, we model  the advertiser’s binary feedback using a nonlinear function $f_{\text{true}}(\cdot):=(\Phi\circ\sin)(\cdot)$ where \(\Phi(\cdot)\) denotes the standard Gaussian CDF. Then, the binary feedback signal is generated by evaluating
\begin{align}
\text{observe}(b_i) = \mathbb{I}\left[\text{Uniform}(0,1) < f_{\text{true}}(b_i)\right],
\label{def:observe}
\end{align}
where $\mathbb{I}$ denotes the indicator function. We use \eqref{def:observe} in line 4 of Algorithm \ref{algoirithm1}. 

We define the daily budget allocation $\xi(t)$ as a piecewise constant function of time \(t\), where \(t\) is measured in \emph{hours}:
\[
\xi(t) :=
\begin{cases}
135, & t \leq 72 \; \text{(Day 1--3)}, \\
300, & 72 < t \leq 120 \; \text{(Day 4--5)}, \\
80, & 120 < t \leq 216 \; \text{(Day 6--9)}, \\
200, & 216 < t \leq 288 \; \text{(Day 10--12)}, \\
130, & 288 < t \leq 384 \; \text{(Day 13--16)}, \\
400, & t > 384 \; \text{(Day 17 onward)}.
\end{cases}
\]

\noindent The seasonality variation in auction dynamics~\cite{he2019identification},  \cite{ karlsson2014adaptive} is simulated using a seasonal function
\[
h(t) := \sin\left(\frac{2\pi t}{T_{\text{day}}} - 1.6\right) + 0.32 \sin\left(\frac{4\pi t}{T_{\text{day}}} - 1.5\right),
\]
where \(T_{\text{day}} = 720\) is the number of steps for one day. We use this $h(t)$ in the unshaded bid computation as in \cite{karlsson2024optimization}.

We define the parameter space of shading variable \(\bm{\theta}=(\theta_1, \theta_2)\in\mathbb{R}^{2}\) as \( [0, 1] \times [0, 3] \), discretized into a uniform \(10 \times 10\) grid, resulting in \(N = 100\) samples $\{\bm{\theta}_j\}_{j\in [N]}\equiv\{(\theta_{1j},\theta_{2j})\}_{j=\in [N]}$ (we use $j$ for sample index). The initial distribution \( \bm{\mu}^0 \in \Delta^{N-1} \) is chosen to be uniform, i.e., $\mu^0 = \frac{1}{N}\sum_{j=1}^{N}\delta_{\bm{\theta}_{j}}$, where $\delta_{\bm{\theta}_{j}}$ is Dirac delta at the $j$th sample $\bm{\theta}_j\in\mathbb{R}^2$, $j\in[N]$.

Win probability $\hat{f}_k$ is computed using~\eqref{eq:winprob} with fixed $\bm{a} = (a_1, a_2, a_3, a_4, a_5)$ representing device type, ad slot position, publisher domain, content category, and user location. For simplicity, we fix only five attributes, but this can be generalized for more as needed. In the simulation, each component of $\bm{a}$ is sampled independently from a uniform distribution on $[0,1]$. The parameters in \eqref{eq:winprob} are estimated via exponentially weighted exponentially weighted logistic regression with \(\ell_2\) regularization, solved by Adam optimizer with learning rate \(10^{-3}\).


We computed the entropy-regularized Wasserstein-proximal updates as per Theorem~\ref{Thm:WassProxOfLinear} with step size \( h = 1 \) and regularization parameter \( \varepsilon = 0.8 \). 


Fig.~\ref{fig:finalBid} shows the evolution of bid prices over the campaign horizon, where the red curve captures adaptations to changing market conditions, seasonality, and budget constraints. Bids rise with higher budgets or win probabilities, and fall under tighter limits or lower demand. 

Fig.~\ref{fig:MuKDistributions} depicts the evolution of the parameter distribution \(\bm{\mu}^k\) over time, computed using our Algorithm \ref{algoirithm1}. The $\bm{\mu}^{k}$ is seen to move from the initial uniform to being concentrated around high-value, high-success regions.

Fig.~\ref{fig:CumDailySpend} compares the desired and actual daily spend governed by \eqref{actualSpend}, showing that the learned policy tracks the non-stationary budget target while following the seasonality and budget constraint trends despite stochastic feedback. Around abrupt budget changes, the policy may briefly fluctuate, but it self-corrects and realigns with the desired trajectory by the following day. Fig.~\ref{fig:CumulativeSpend} shows that cumulative spend (dashed red) closely follows the target (solid black).

\begin{figure}[t]
    \centering
 \includegraphics[width=\linewidth]{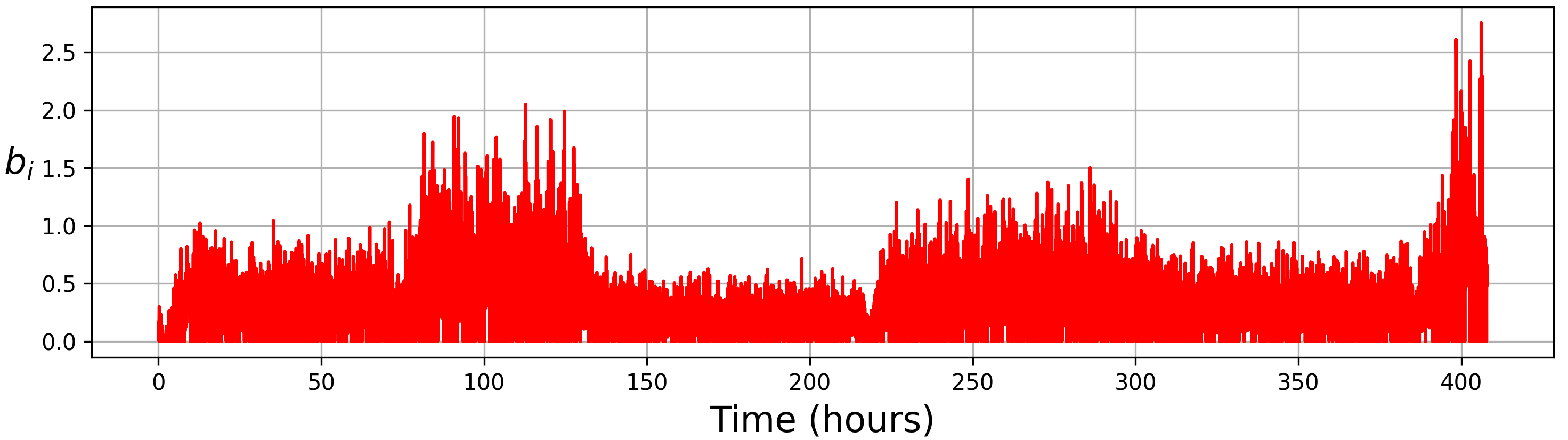}
    \caption{Final submitted bid $b_i$ over campaign time.}
    \vspace*{-0.1in}
    \label{fig:finalBid}
\end{figure}

\begin{figure}[ht]
    \centering
    \includegraphics[width=\linewidth]{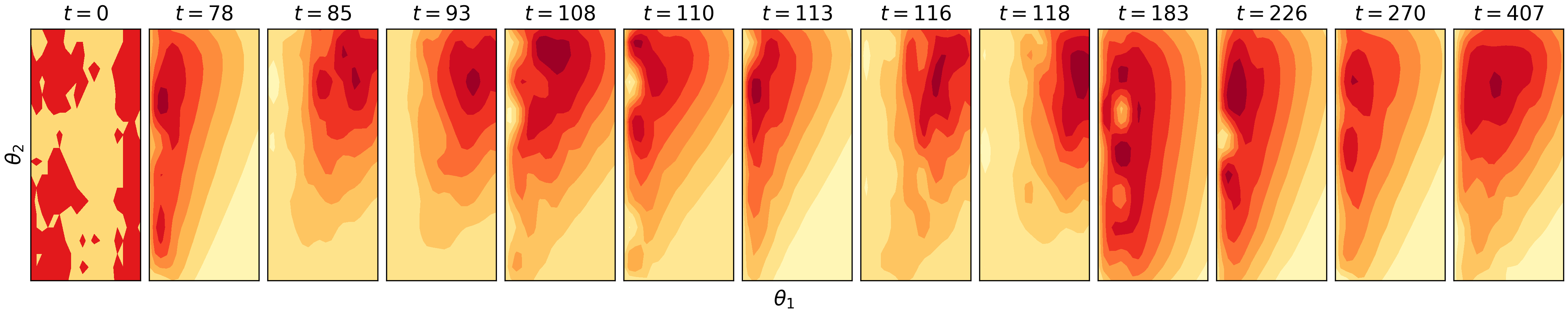}
    \caption{Evolution of the  distribution \(\bm{\mu}^k\) over the bid shading parameter $(\theta_1,\theta_{2})$ space $[0, 1] \times [0, 3]$.}
    \label{fig:MuKDistributions}
\end{figure}

\begin{figure}[ht]
    \centering
    \begin{subfigure}[t]{0.48\linewidth}
        \centering
        \includegraphics[width=\linewidth]{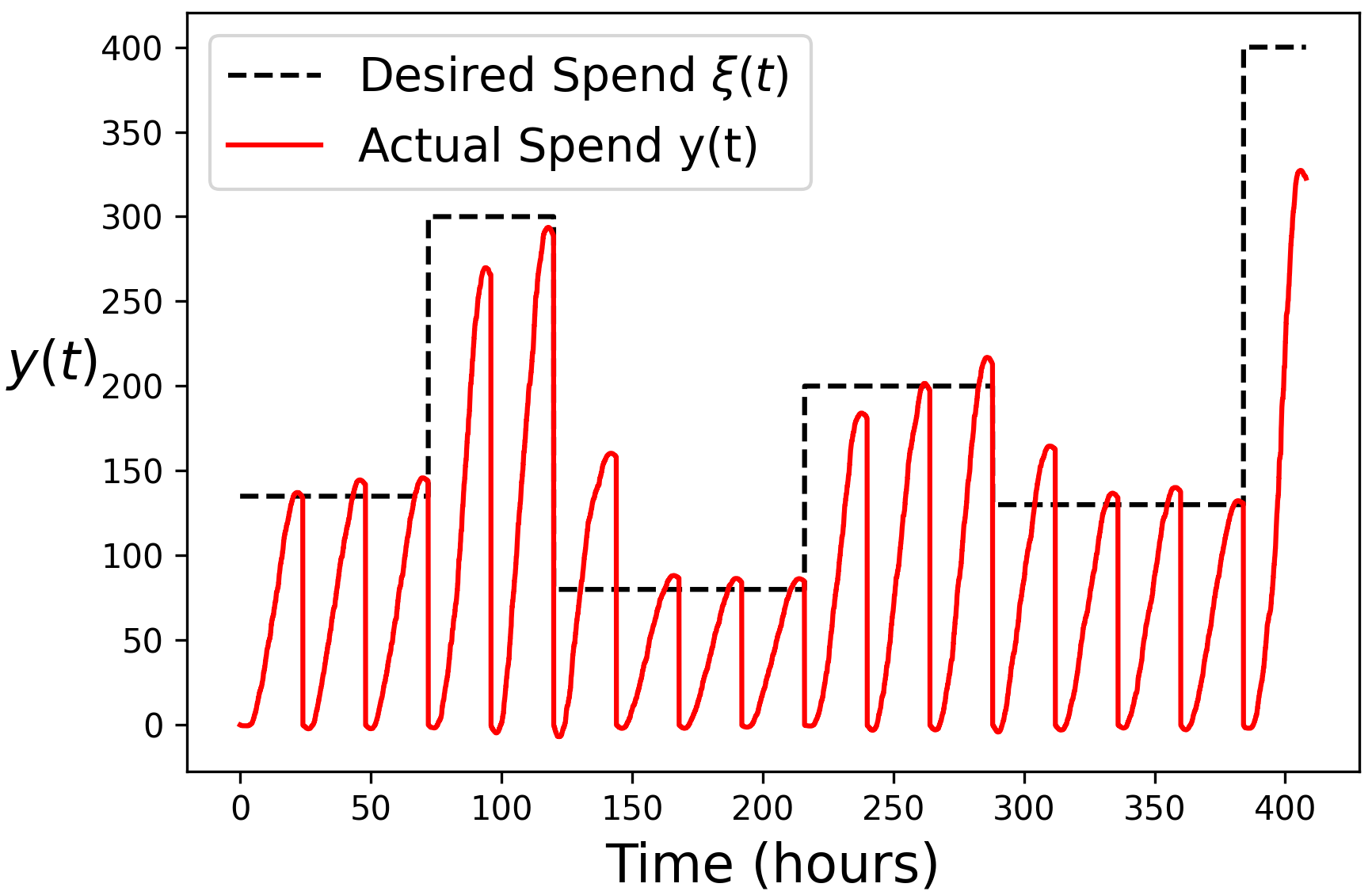}
        \caption{Desired \(\xi(t)\) vs. actual spend $y(t)$ over time $t$.}
        \label{fig:CumDailySpend}
    \end{subfigure}
    \hfill
    \begin{subfigure}[t]{0.48\linewidth}
        \centering
        \includegraphics[width=\linewidth]{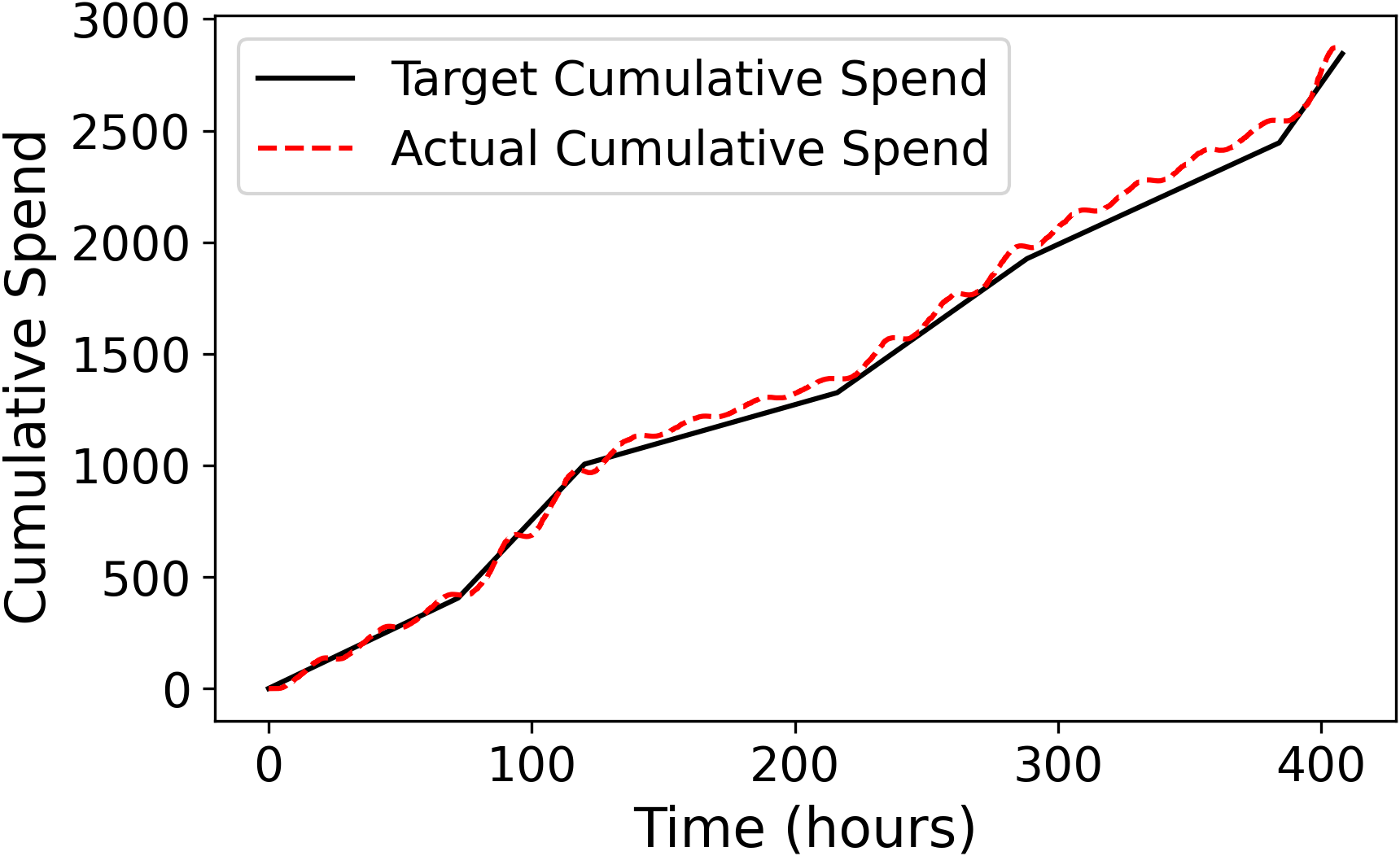}
        \caption{Cumulative target vs. actual spend over time.}
        \label{fig:CumulativeSpend}
    \end{subfigure}
    \caption{Comparison of daily and cumulative spending.}
    \label{fig:SpendComparison}
\end{figure}




\section{Conclusion}\label{sec:conclusions}
We propose a new approach to bid shading in first-price auctions by formulating the problem as a convex optimization over the space of joint probability distribution of shading parameters. The proposed method leverages a closed form entropy-regularized Wasserstein-proximal update that adapts the parameter distribution based on observed auction outcomes. An illustrative numerical experiment demonstrates the proposed solution.


\bibliographystyle{IEEEtran}
\bibliography{references.bib}

\end{document}